\documentclass[10pt,twocolumn,letterpaper]{article}

\usepackage{cvpr}
\usepackage{times}
\usepackage{epsfig}
\usepackage{graphicx}
\usepackage{amsmath}
\usepackage{amssymb}

\usepackage{overpic}
\usepackage{amssymb}
\usepackage{amsmath}
\usepackage{amsthm}
\usepackage{graphicx}
\usepackage{booktabs} 
\usepackage{multirow} 
\usepackage{diagbox}
\usepackage{titling} 
\usepackage[T1]{fontenc} 
\usepackage[style=american]{csquotes} 

\usepackage{color}
\definecolor{turquoise}{cmyk}{0.65,0,0.1,0.3}
\definecolor{purple}{rgb}{0.65,0,0.65}
\definecolor{dark_green}{rgb}{0, 0.5, 0}
\definecolor{orange}{rgb}{0.8, 0.6, 0.2}
\definecolor{red}{rgb}{0.8, 0.2, 0.2}
\definecolor{blueish}{rgb}{0.0, 0.7, 1}
\definecolor{light_gray}{rgb}{0.7, 0.7, .7}
\definecolor{pink}{rgb}{1, 0, 1}
\definecolor{dark_red}{rgb}{0.5, 0, 0}

\newcommand{\hide}[1]{{}} 

\newcommand{\Figure}[1]{Figure~\ref{fig:#1}}

\newcommand{\Table}[1]{Table~\ref{tbl:#1}}
\newcommand{\eq}[1]{(\ref{eq:#1})}

\newcommand{\Section}[1]{Section~\ref{sec:#1}}

\usepackage{currfile}

\newcommand{\drawmyfilename}{\begin{flushleft}\begin{tikzpicture}[overlay,rotate=90,transform shape] \protect\filldraw[black] (0,+.1in) circle (0pt) node[anchor=west] {\color{red} [\currfilename{}]}; \end{tikzpicture} \end{flushleft}\vspace{-\baselineskip}}
\renewcommand{\drawmyfilename}{\vspace{-.25\baselineskip}} 

\usepackage{lipsum}

\usepackage{parskip}
\renewcommand{\paragraph}[1]{\vspace{1.5\parskip}\textbf{#1}.}

\DeclareMathAlphabet\mathbfcal{OMS}{cmsy}{b}{n}

\newcommand{\CIRCLE}[1]{\raisebox{.5pt}{\footnotesize \textcircled{\raisebox{-.6pt}{#1}}}}

\newtheorem{observation}{Observation}
\DeclareMathOperator{\proj}{Proj}
\newcommand{\SIF}{SIF~\cite{genova_iccv19}}
\newcommand{\VP}{VP~\cite{tulsiani_cvpr17}}
\newcommand{\OccNet}{OccNet~\cite{occnet_cvpr19}}
\newcommand{\PtoM}{P2M~\cite{wang_eccv18}}
\newcommand{\AtlasNet}{AtlasNet~\cite{atlasnet_cvpr18}}
\newcommand{\SQ}{SQ~\cite{paschalidou_cvpr19}}

\newcommand{\loss}[1]{\mathcal{L}_\text{#1}}

\newcommand{\R}{\mathbb{R}}
\newcommand{\expect}{\mathbb{E}}
\newcommand{\bT}{\mathbf{T}}

\newcommand{\bo}{\mathbf{o}}
\newcommand{\latent}{\boldsymbol{\lambda}}

\newcommand{\object}{\mathcal{O}}
\newcommand{\union}{\mathcal{U}}
\newcommand{\sdf}{\Phi}
\newcommand{\convex}{\mathcal{C}}
\newcommand{\pars}{{\boldsymbol{\beta}}}

\newcommand{\x}{\mathbf{x}}
\newcommand{\n}{\mathbf{n}}
\newcommand{\halfspace}{\mathcal{H}}
\newcommand{\encoder}{\mathcal{E}}
\newcommand{\groupencoder}{\mathcal{E}}
\newcommand{\groupdecoder}{\mathcal{D}}
\newcommand{\decoder}{\mathcal{D}}

\newcommand{\slope}{\sigma}
\newcommand{\smooth}{\delta}
\newcommand{\sigmoid}{\text{{Sigmoid}}}
\newcommand{\softmax}{\text{{LogSumExp}}}

\usepackage[pagebackref=true,breaklinks=true,colorlinks,bookmarks=false]{hyperref}

\cvprfinalcopy 

\def\cvprPaperID{5908} 

\ifcvprfinal\fi
\begin{document}

\title{\textbf{{\huge CvxNet:} \\ Learnable Convex Decomposition} \\[1em]}

\author{%
Boyang Deng \\ Google Research \and
Kyle Genova \\ Google Research \and 
Soroosh Yazdani \\ Google Hardware \and
Sofien Bouaziz \\ Google Hardware \and
Geoffrey Hinton \\ Google Research \and
Andrea Tagliasacchi \\ Google Research
} 
\date{}

\maketitle
\pagestyle{plain} 

\begin{abstract}
Any solid object can be decomposed into a collection of convex polytopes (in short, convexes).
When a small number of convexes are used, such a decomposition can be thought of as a piece-wise approximation of the geometry.
This decomposition is fundamental in computer graphics, where it provides one of the most common ways to approximate geometry, for example, in real-time physics simulation.
A convex object also has the property of being simultaneously an explicit and implicit representation: one can interpret it \textit{explicitly} as a mesh derived by computing the vertices of a convex hull, or \textit{implicitly} as the collection of half-space constraints or support functions.
Their implicit representation makes them particularly well suited for neural network training, as they abstract away from the topology of the geometry they need to represent.
However, at testing time, convexes can also generate explicit representations -- polygonal meshes -- which can then be used in any downstream application.
We introduce a network architecture to represent a low dimensional family of convexes. This family is automatically derived via an auto-encoding process.
We investigate the applications of this architecture including automatic convex decomposition, image to 3D reconstruction, and part-based shape retrieval.
\end{abstract}

\begin{figure}[t]
\centering
\begin{overpic} 
[width=\linewidth]
{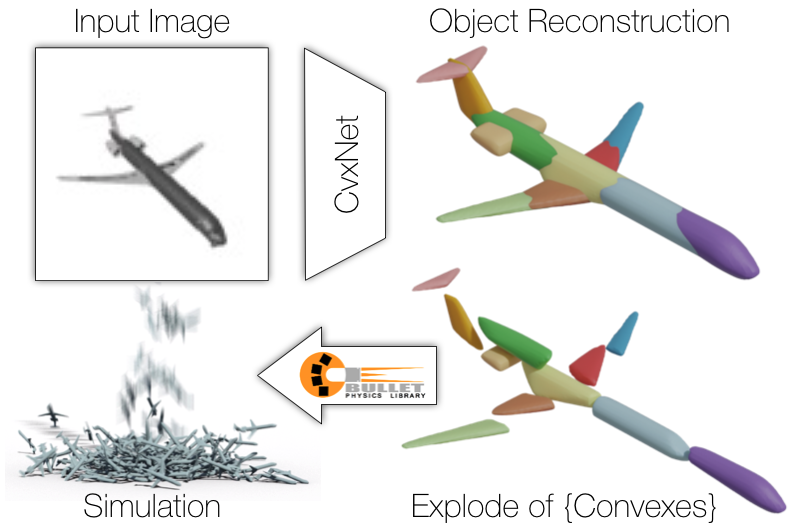}
\end{overpic}
\drawmyfilename
\caption{
Our method reconstruct a 3D object from an input image as a collection of convex hulls, and we visualize the explode of these convexes.
Notably, CvxNet outputs polygonal mesh representations of convex polytopes \textit{without} requiring the execution of computationally expensive iso-surfacing~(e.g. Marching Cubes).
This means the representation outputted by CvxNet can then be readily used for physics simulation~\cite{bullet}, as well as many other downstream applications that consume polygonal meshes.
}
\label{fig:teaser}
\end{figure}


\section{Introduction}
While images admit a standard representation in the form of a scalar function uniformly discretized on a grid, the curse of dimensionality has prevented the effective usage of analogous representations for learning 3D geometry.
%
Voxel representations have shown some promise at low resolution \cite{brock_arxiv16, gadelha_17, liao_cvpr18, rezende2016unsupervised, stutz2018learning, ulusoy_15, wu2016learning}, while hierarchical representations have attempted to reduce the memory footprint required for training~\cite{riegler2017octnet,tatarchenko2017octree,wang2018adaptive}, but at the significant cost of complex implementations.
%
Rather than representing the \textit{volume} occupied by a 3D object, one can resort to modeling its \textit{surface} via a collection of points~\cite{achlioptas_icml17,fan2017point}, polygons~\cite{kanazawa_eccv18,ranjan_eccv18,wang_eccv18}, or surface patches~\cite{atlasnet_cvpr18}.
%
Alternatively, one might follow Cezanne's advice and ``treat nature by means of the cylinder, the sphere, the cone, everything brought into proper perspective'', and think to approximate 3D geometry as~\textit{geons}~\cite{biederman_87} -- collections of simple to interpret geometric primitives~\cite{tulsiani_cvpr17,zou_cvpr17}, and their composition~\cite{csgnet_cvpr18,genova_iccv19}.
%
Hence, one might rightfully start wondering \textit{``why so many representations of 3D data exist, and why would one be more advantageous than the other?''}
One observation is that multiple equivalent representations of 3D geometry exist because real-world applications need to perform different \textit{operations} and queries on this data (~\cite[Ch.1]{pmp}).
%
For example, in computer graphics, points and polygons allow for very efficient rendering on GPUs, while volumes allow artists to sculpt geometry without having to worry about tessellation~\cite{zbrush} or assembling geometry by smooth composition~\cite{angles2017sketch}, while primitives enable highly efficient collision detection~\cite{thul_sig18} and resolution~\cite{tkach_siga16}.
%
In computer vision and robotics, analogous trade-offs exist: surface models are essential for the construction of low-dimensional parametric templates essential for  tracking~\cite{blanz_sig99,bogo_cvpr16}, volumetric representations are key to capturing 3D data whose topology is unknown~\cite{kinfu,dynfu}, while part-based models provide a natural decomposition of an object into its semantic components. Part-based models create a representation useful to reason about extent, mass, contact, ... quantities that are key to describing the scene, and planning motions~\cite{diffrigidbody,realtosim}.

\paragraph{Contributions}
In this paper, we propose a novel representation for geometry based on primitive decomposition. The representation is parsimonious, as we \textit{approximate} geometry via a \textit{small number} of \textit{convex} elements, while we seek to allow low-dimensional representation to be automatically inferred from data -- without any human supervision.
More specifically, inspired by recent works~\cite{tulsiani_cvpr17,genova_iccv19,occnet_cvpr19} we train our pipeline in an unsupervised
manner: predicting the primitive configuration as well as their parameters by checking whether the reconstructed geometry matches the geometry of the target.
%
We note how we inherit a number of interesting \textit{properties} from several of the aforementioned representations.
As it is part-based it is naturally locally supported, and by training on a shape collection, parts have a semantic association (i.e. the same element is used to represent the backs of chairs).
Although part-based, each of them is not restricted to belong to the class of boxes~\cite{tulsiani_cvpr17}, ellipsoids~\cite{genova_iccv19}, or sphere-meshes~\cite{tkach_siga16}, but to the more general class of convexes.
As a convex is defined by a collection of half-space constraints, it can be simultaneously decoded into an explicit (polygonal mesh), as well as implicit (indicator function) representation.
Because our encoder decomposes geometry into convexes, it is immediately usable in any application requiring real-time physics simulation, as collision resolution between convexes is efficiently decided by GJK~\cite{gjk} (\Figure{teaser}). 
Finally, parts can interact via structuring~\cite{genova_iccv19} to generate smooth blending between parts.

\section{Related works}

One of the simplest high-dimensional representations is voxels, and they are the most commonly used representation for discriminative \cite{voxnet2015, qi2016volumetric, song2016deep}
models, due to their similarity to image based convolutions. 
Voxels have also been used successfully for generative models \cite{wu20153d, choy20163d, girdhar2016learning, rezende2016unsupervised, stutz2018learning, wu2016learning}. However, the memory requirements of voxels makes them unsuitable for resolutions larger than $64^3$. One can reduce the memory consumption significantly by using octrees that take advantage of the sparsity of voxels \cite{riegler2017octnet, wang2017cnn, wang2018adaptive, tatarchenko2017octree}. This can extend the resolution to $512^3$, for instance, but comes at the cost of more complicated implementation.

\begin{figure*}[th]
\centering
\begin{overpic} 
[width=\linewidth]
{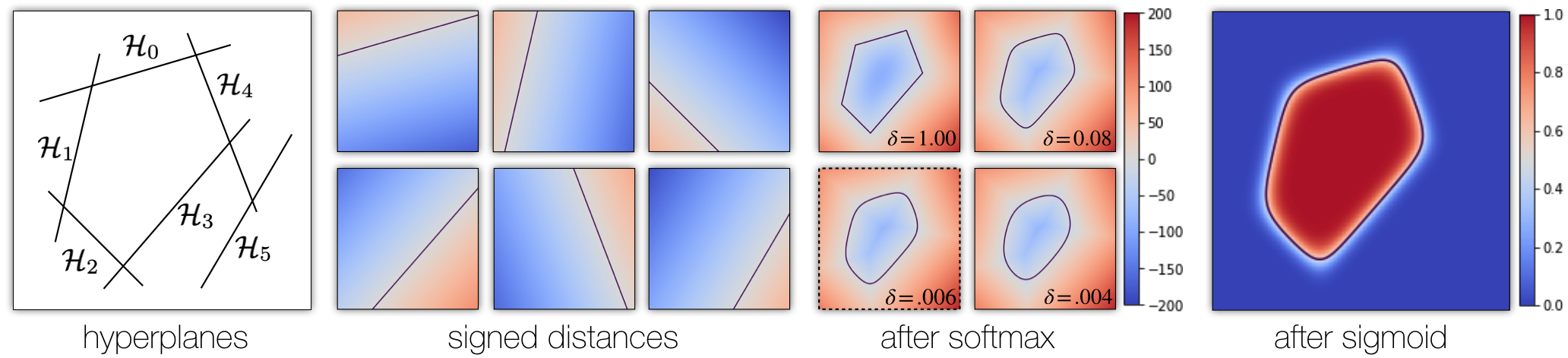}
\end{overpic}
\drawmyfilename{}
\caption{
\textbf{From \{hyperplanes\} to occupancy} --
A collection of hyperplane parameters for an image specifies the indicator \textit{function} of a convex.
The soft-max allows gradients to propagate through all hyperplanes and allows for the generation of \textit{smooth} convex, while the sigmoid parameter controls the \textit{slope} of the transition in the generated indicator -- note that our soft-max function is a LogSumExp. 
%
}
\label{fig:cvxdec}
\end{figure*}

\newpage
\paragraph{Surfaces}
In computer graphics, \textit{polygonal meshes} are the standard representation of 3D objects.
Meshes have also been considered for discriminative classification by applying graph convolutions to the mesh \cite{masci2015geodesic, bronstein2017geometric, guo20153d, monti2017geometric}. 
Recently, meshes have also been considered as the output of a network \cite{atlasnet_cvpr18, kong2017using, wang_eccv18}.
A key weakness of these models is the fact that they may 
produce self-intersecting meshes. 
Another natural high-dimensional representation that has garnered some traction in vision is the \textit{point cloud} representation. Point clouds are the natural representation of objects if one is using sensors such as depth cameras or LiDAR, and they require far less memory than voxels.
Qi et al. \cite{qi2017pointnet, qi2017pointnet++} used point clouds as a representation for discriminative deep learning tasks.
Hoppe et al. \cite{hoppe1992} used point clouds for surface mesh reconstruction (see also \cite{berger2017} for a survey
of other techniques).
Fan et. al. \cite{fan2017point} and Lin et. al. \cite{lin2018learning} used point clouds for 3D reconstruction using deep
learning. 
However, these approaches require additional non-trivial post-processing steps to generate the final 3D mesh.

\paragraph{Primitives}
Far more common is to approximate the input shape by set of volumetric primitives. With this perspective in mind, representing shapes as voxels will be a special case, where the primitives are unit cubes in a lattice. 
Another fundamental way to describe 3D shapes is via \textit{Constructive Solid Geometry}~\cite{laidlaw1986constructive}.
Sherma~et.~al.~\cite{csgnet_cvpr18} presents a model that will output a program (i.e. set of Boolean operations on shape primitives) that generate the input image or shape.
In general, this is a fairly difficult task.
Some of the classical primitives used in graphics and computer vision are blocks world \cite{roberts1963machine}, 
generalized cylinders \cite{binford1971visual}, geons \cite{biederman_87}, and even 
Lego pieces \cite{van2015part}. In \cite{tulsiani_cvpr17}, a deep CNN is used to interpret a
shape as a union of simple rectangular prisms. They also note that their model provides a
consistent parsing across shapes (i.e. the head is captured by the same primitive), allowing
some interpretability of the output.
In \cite{paschalidou_cvpr19}, they extended cuboids to {\em superquadrics}, showing that the extra flexibility will result in better reconstructions.

\paragraph{Implicit surfaces}
If one generalizes the shape primitives to analytic surfaces (i.e. level sets of analytic
functions), then new analytic tools become available for generating shapes. 
In \cite{occnet_cvpr19,imnet}, for instance, they train a model to discriminate inside coordinates from outside
coordinates (referred to as an {\em occupancy function} in the paper, and as an \textit{indicator function}
in the graphics community). Park et. al. \cite{park2019deepsdf} used the signed distance function to the
surface of the shape to achieve the same goal.
One disadvantage of the implicit
description of the shape is that most of the interpretability is missing from the final answer.
In \cite{genova_iccv19}, they take a more geometric approach and restrict to level sets of
axis-aligned Gaussians. 
Partly due to the restrictions of these functions, their representation struggles on shapes with angled parts, but they do recover the interpretability that~\cite{tulsiani_cvpr17} offers. 

\paragraph{Convex decomposition}
In graphics, a common method to represent shapes is to describe them as a collection of convex objects. Several methods for convex decomposition of meshes have been proposed~\cite{graham1972efficient, preparata1977convex}.
In machine learning, however, we only find early attempts to approach convex hull computation via neural networks~\cite{leung1997neural}.
Splitting the meshes into exactly convexes generally 
produces too many pieces \cite{chazelle1981convex}.
As such, it is more prudent to seek small number of convexes  that \textit{approximate} the input shape~\cite{ghosh2013fast, lien2007approximate, liu2016nearly, mamou2016volumetric, mamou2009simple}.
Recently \cite{thul_sig18} also extended convex decomposition to the spatio-temporal domain, by considering moving geometry.
Our method is most related to \cite{tulsiani_cvpr17} and \cite{genova_iccv19}, in that we train an occupancy function.
However, we choose our space of functions so that their level sets are approximately convex, and use these as building blocks.

\begin{figure*}
\centering
\begin{overpic} 
[width=\linewidth]
{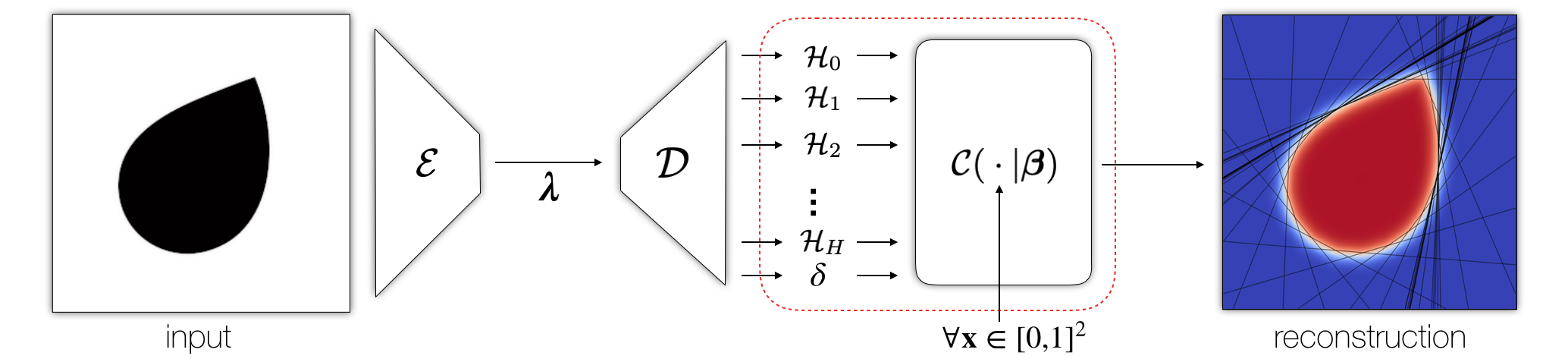}
\end{overpic}
\drawmyfilename{}
\caption{
\textbf{Convex auto-encoder} --
%
The encoder~$\encoder$ creates  a low dimensional latent vector representation $\lambda$, decoded into a collection of hyperplanes by the decoder~$\decoder$. The training loss involves reconstructing the value of the input image at random pixels~$\x$.
}
\label{fig:cvxfam}
\end{figure*}

\section{Method -- CvxNet}
Our object is represented via an indicator $\object: \R^3\! \rightarrow\![0,1]$, and with $\partial\object = \{\x \in \R^3\:|\:\object(x)=0.5\}$ we indicate the surface of the object.
The indicator function is defined such that $\{\x \in \R^3\:|\:\object(x) = 0\}$ defines the outside of the object and $\{\x \in \R^3\:|\:\object(x) = 1\}$ the inside.
Given an input ({\it e.g.} an image, point cloud, or voxel grid) an encoder estimates the parameters $\{\pars_k\}$ of our template representation $\hat\object(\cdot)$ with $K$ primitives (indexed by $k$).
We then evaluate the template at random sample points $\x$, and our training loss ensures $\hat\object(\x) \approx \object(\x)$.
In the discussion below, without loss of generality, we use 2D illustrative examples where $\object: \R^2\! \rightarrow\![0,1]$.
Our representation is a \textit{differentiable convex decomposition}, which is used to train an image encoder in an end-to-end fashion.
We begin by describing a differentiable representation of a single convex object~(\Section{cvxdec}).
Then we introduce an auto-encoder architecture to create a low-dimensional family of approximate convexes~(\Section{cvxfam}).
These allow us to represent objects as spatial compositions of convexes (\Section{multicvx}).
We then describe the losses used to train our networks~(\Section{losses}) and mention a few implementation details (\Section{implementation}).

\subsection{Differentiable convex indicator -- \Figure{cvxdec}}
\label{sec:cvxdec}
We define a decoder that given a collection of (unordered) half-space constraints constructs the indicator function of a single convex object; such a function can be evaluated at any point $\x \in \mathbb{R}^3$.
We define $\halfspace_h(\x)=\n_h\cdot\x+d_h$ as the signed distance of the point $\x$ from the $h$-th
plane with normal $\n_h$ and offset $d_h$.
Given a sufficiently large number $H$ of half-planes the signed distance function of any convex object can be approximated by taking the $\max$ of the signed distance functions of the planes.
To facilitate gradient learning, instead of maximum, we use the smooth maximum function $\softmax$ and define the \textit{approximate} signed distance function, $\sdf(x)$: 
\begin{equation}
    \sdf(\x) = \softmax\{ \smooth \halfspace_h(\x) \},
\end{equation}
Note this is an approximate SDF, as the property $\|\nabla \sdf(\x)\| \!=\!1$ is not necessarily satisfied $\forall \x$.
We then convert the signed distance function to an indicator function~$\convex: \R^3\! \rightarrow\![0,1]$:

\begin{equation}
    \convex(\x|\pars) = \sigmoid( - \slope \sdf(\x)),
    \label{eq:indicator}
\end{equation}
%
We denote the collection of hyperplane parameters as $\mathbf{h}=\{(\mathbf{n}_h, d_h) \}$, and the overall set of parameters for a convex as $\pars = [\mathbf{h}, \slope]$.
We treat $\slope$ as a hyperparameter, and consider the rest as the learnable parameters of our representation.
As illustrated in \Figure{cvxdec}, the parameter $\smooth$ controls the smoothness of the generated convex, while $\slope$ controls the sharpness of the transition of the indicator function. Similar to the smooth maximum function, the soft classification boundary created by $\sigmoid$ facilitates training.
 
In summary, given a collection of hyperplane parameters, this differentiable module generates a function that can be evaluated at any position $\x$.

\subsection{Convex encoder/decoder -- \Figure{cvxfam}}
\label{sec:cvxfam}
A sufficiently large set of hyperplanes can represent any convex object, but one may ask whether it would be possible to discover some form of correlation between their parameters.
Towards this goal, we employ an auto-encoder architecture illustrated in \Figure{cvxfam}.
Given an input, the encoder $\encoder$ derives a bottleneck representation $\latent$ from the input.
Then, a decoder $\decoder$ derives the collection of hyperplane parameters.
While in theory permuting the $H$ hyperplanes generates the same convex, the decoder $\decoder$ correlates a particular hyperplane with a corresponding orientation.
This is visible in \Figure{correlation}, where we color-code different 2D hyperplanes and indicate their orientation distribution in a simple 2D auto-encoding task for a collection of axis-aligned ellipsoids.
As ellipsoids and oriented cuboids are convexes, we argue that the architecture in \Figure{cvxfam} allows us to generalize the core geometric primitives proposed in \VP{} and \SIF{}; we verify this claim in \Figure{interp2d}.


\begin{figure}[h]
\centering
\begin{minipage}[c]{.4\linewidth}
  \begin{overpic} 
    [width=\linewidth]
    {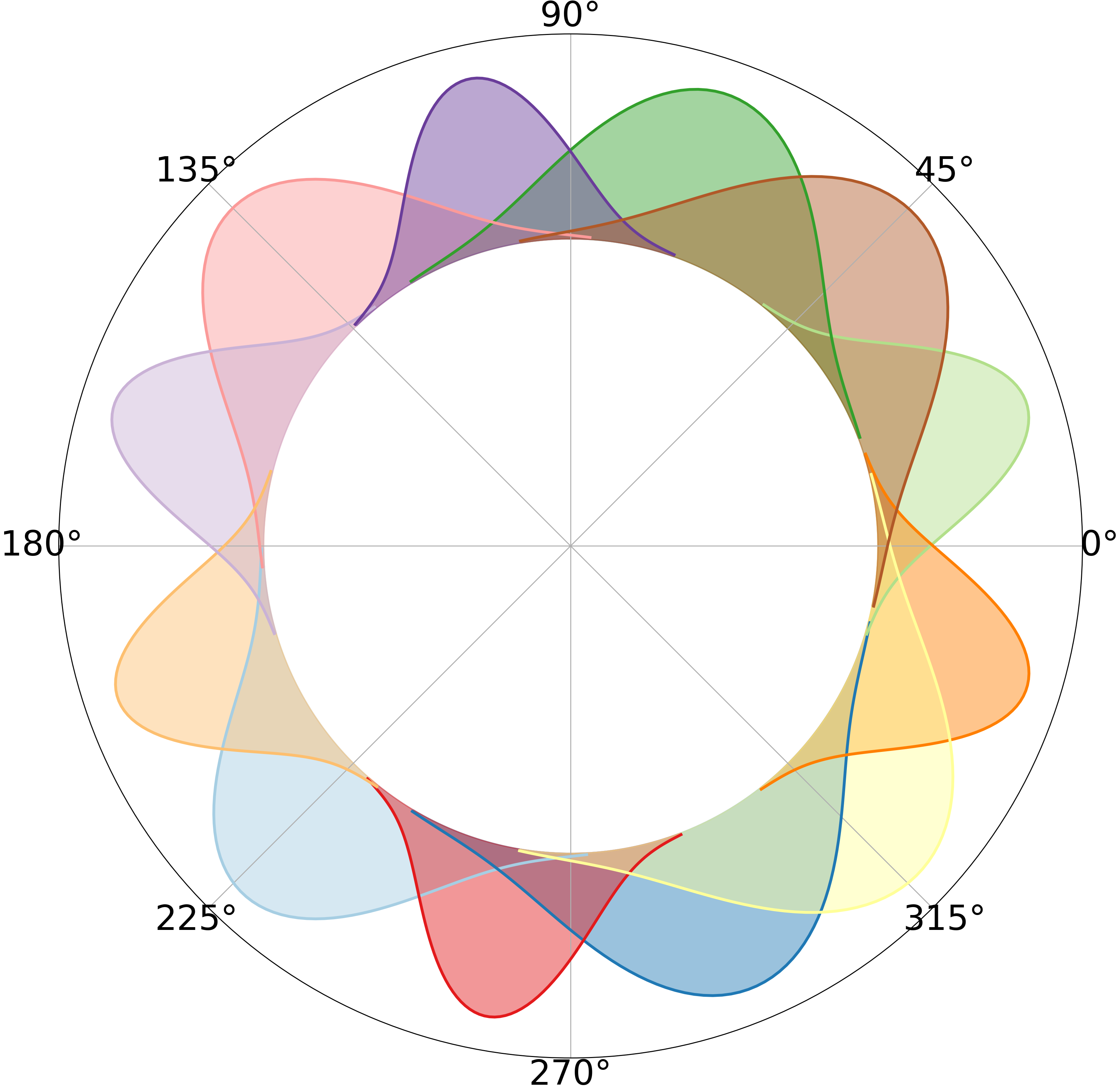}
    \end{overpic}
\end{minipage}
\begin{minipage}[c]{0.58\linewidth}
\caption{
\textbf{Correlation} --
While the description of a convex, $\{(\mathbf{n}_h, \mathbf{d}_h) \}$, is permutation invariant we employ an encoder/decoder that implicitly establishes an ordering.
Our visualization reveals how a particular hyperplane typically represents a particular subset of orientations.
}
\label{fig:correlation}
\end{minipage}
\end{figure}
\begin{figure}[h]
\centering
\begin{minipage}[c]{.4\linewidth}
  \begin{overpic} 
    [width=\linewidth]
    {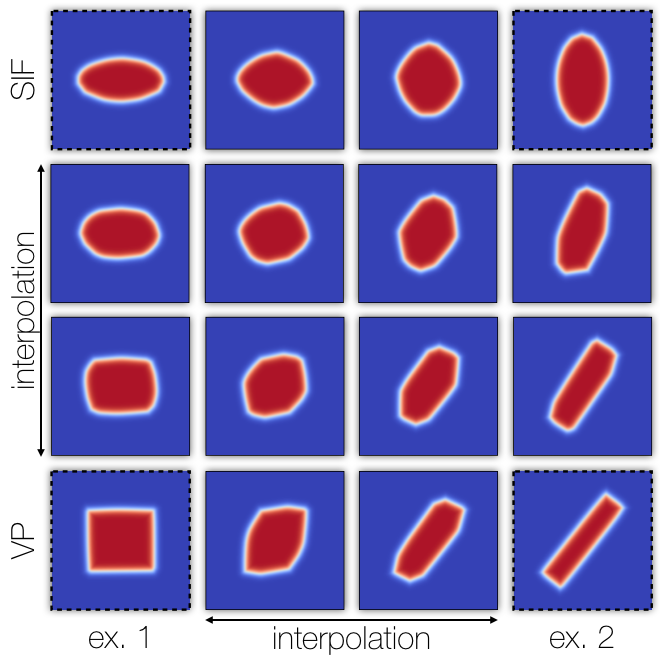}
    \end{overpic}
\end{minipage}
\begin{minipage}[c]{0.58\linewidth}
\caption{
\textbf{Interpolation} --
%
We compute latent code of shapes in the corners using CvxNet. We then linearly interpolate latent codes to synthesize shapes in-between.
Our primitives generalize the shape space of \VP{} (boxes) and \SIF{} (ellipsoids) so we can interpolate between them smoothly.
}
\label{fig:interp2d}
\end{minipage}
\end{figure}

\subsection{Explicit interpretation -- \Figure{simulation}}
\label{sec:explicit}
What is significantly different from other methods that employ indicator functions as trainable representations of 3D geometry, is that convexes generated by our network admit an \textit{explicit} interpretation: they can be easily converted into polygonal meshes.
This is in striking contrast to \cite{park2019deepsdf,imnet,genova_iccv19,occnet_cvpr19}, where a computationally intensive iso-surfacing operation needs to be executed to extract their surface~(e.g.~Marching Cubes~\cite{marchingcubes}).
More specifically, iso-surfacing techniques typically suffer the curse of dimensionality, with a performance that scales as $1/\varepsilon^d$, where $\varepsilon$ the desired spatial resolution and $d$ is usually $3$. 
Conversely, as we illustrate in \Figure{simulation}, we only require the execution of two duality transforms, and the computations of two convex hulls of $H$ points.
The complexity of these operations is clearly \textit{independent} of any resolution parameter~$\varepsilon$.
\begin{figure}[h]
\centering
\begin{overpic} 
[width=\linewidth]
{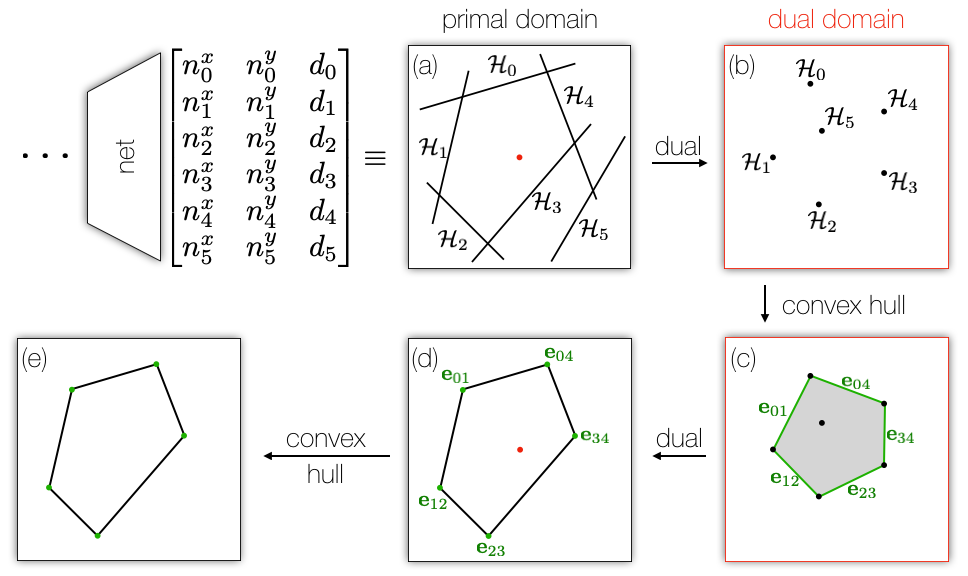}
\end{overpic}
\drawmyfilename{}
\caption{
\textbf{From \{hyperplanes\} to polygonal meshes} --
The polygonal mesh corresponding to a set of hyperplanes (a) can be computed by transforming planes into points via a duality transform (b), the computation of a convex hull (c), a second duality transform (d), and a final convex hull execution (e). The output of this operation is a \textit{polygonal mesh}.
Note this operation is efficient, output sensitive, and, most importantly does not suffer the curse of dimensionality. 
Note that, for illustration purposes, the duality coordinates in this figure are fictitious.
}
\label{fig:simulation}
\end{figure}

%
%

\begin{figure*}[ht]
\centering
\begin{overpic} 
[width=\linewidth]
{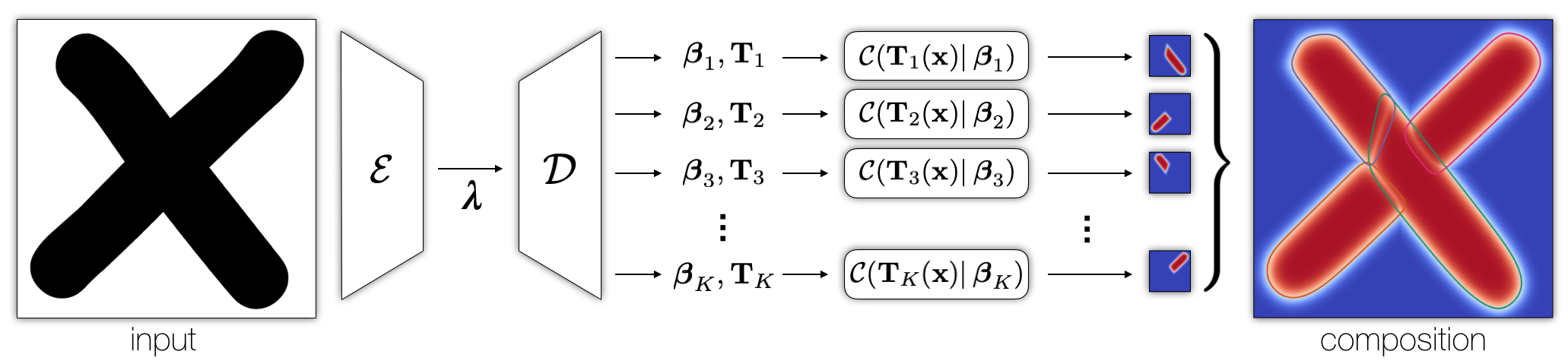}
\end{overpic}
\drawmyfilename{}
\caption{
\textbf{Multi-convex auto-encoder} --
Our network approximates input geometry as a \textit{composition} of convex elements. Note that this network \textit{does not} prescribe how the final image is generated, but merely output the shape $\{\pars{}_k\}$ and pose $\{\mathbf{T}_k\}$ parameters of the abstraction.
Note that this is an illustration where the parameters $\{\boldsymbol{\beta}_k\}, \{ \boldsymbol{T}_k \}$ have been directly optimized via SGD with a preset $\smooth$.
%
}
\label{fig:multicvx}
\end{figure*}

\subsection{Multi convex decomposition -- \Figure{multicvx}}
\label{sec:multicvx}
Having a learnable pipeline for a single convex object, we can now expand the expressivity of our model by representing generic \textit{non-convex} objects as \textit{compositions} of convexes~\cite{thul_sig18}.
To achieve this task an encoder $\groupencoder$ outputs a low-dimensional bottleneck representation of all $K$ convexes $\boldsymbol{\lambda}$ that $\groupdecoder$ decodes into a \textit{collection} of $K$ parameter tuples.
Each tuple~(indexed by $k$) is comprised of a shape code $\pars_k$, and corresponding transformation $\mathbf{T}_k(\x)=\x+\mathbf{c}_k$ that transforms the point from world coordinates to local coordinates. $\mathbf{c}_k$ is the predicted translation vector~(\Figure{multicvx}).


\subsection{Training losses}
\label{sec:losses}

First and foremost, we want the~(ground truth) indicator function of our object $\object$ to be well approximated:
\begin{equation}
\mathcal{L}_\text{approx}(\omega) = \expect_{\x \sim \R^3} \| \hat\object(\x) - \object(\x) \|^2,
\label{eq:loss_approx}
\end{equation}
where $\hat\object(\x) = \max_k \{ \convex_k (\x) \}$, and $\convex_k(\x) = \convex(\bT_k(\x) | \pars_k )$.
The application of the $\max$ operator produces a perfect union of convexes. While constructive solid geometry typically applies the $\min$ operator to compute the union of signed distance functions, note that we apply the $\max$ operator to indicator functions instead with the same effect;
see~\Section{union} in the supplementary material for more details.
We couple the approximation loss with several auxiliary losses that enforce the desired properties of our decomposition.

\paragraph{Decomposition loss (auxiliary)}
We seek a parsimonious decomposition of an object akin to   Tulsiani~et~al.~\cite{tulsiani_cvpr17}.
Hence, overlap between elements should be discouraged:
\begin{equation}
\mathcal{L}_\text{decomp}(\omega) = \expect_{\x \sim \R^3} \| \text{relu}(\underset{k}{\text{sum}} \{ \convex_k (\x) \} - \tau) \|^2,
\label{eq:loss_decomp}
\end{equation}
where we use a permissive $\tau\!=\!2$, and note how the ReLU activates the loss only when an overlap occurs.

\begin{figure}[b]
\centering
\begin{overpic} 
[width=\linewidth]
{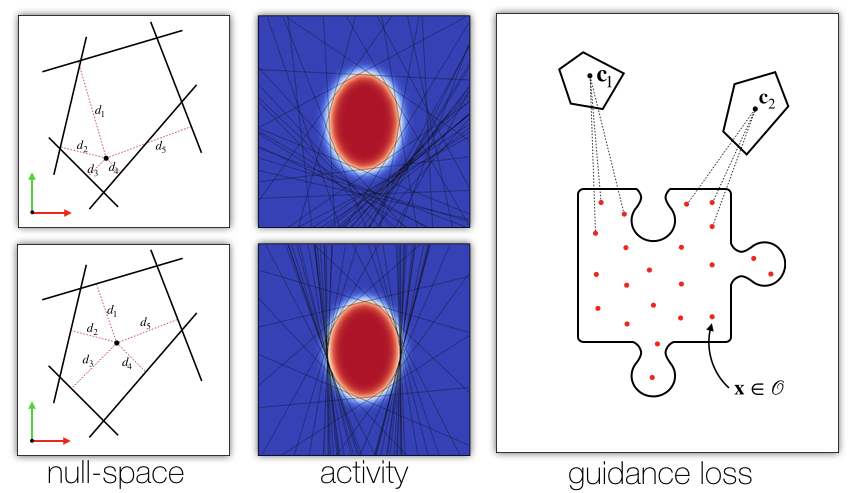}
\end{overpic}
\drawmyfilename{}
\caption{
\textbf{Auxiliary losses -- }
Our $\loss{unique}$ loss (left) prevents the existence of a null-space in the specification of convexes, and (middle) ensures inactive hyperplanes can be easily activated during training.
(right)
Our $\loss{guide}$ move convexes towards the representation of samples drawn from within the object $\x \in \object$.
}
\label{fig:losses}
\end{figure}
\paragraph{Unique parameterization loss (auxiliary)}
While each convex is parameterized with respect to the origin, there is a nullspace of solutions -- we can move the origin to another location within the convex, and update offsets $\{d_h\}$ and transformation $\bT$ accordingly -- see \Figure{losses}(left).
To remove such a null-space, we simply regularize the magnitudes of the offsets for each of the $K$ elements:
\begin{equation}
\loss{unique}(\omega) = \tfrac{1}{H} \sum_{h} \| d_h \|^2
\label{eq:loss_unique}
\end{equation}
In the supplementary material, we prove that minimizing $\loss{unique}$ leads to a unique solution and centers the convex body to the origin. This loss further ensures that \enquote{inactive} hyperplane constraints can be readily re-activated during learning.
Because they fit tightly around the surface they are therefore sensitive to shape changes.

\paragraph{Guidance loss (auxiliary)}
As we will describe in \Section{implementation}, we use offline sampling to speed-up training.
However, this can cause severe issues.
In particular, when a convex \enquote{falls within the cracks} of sampling
(i.e.~$\nexists x \:|\: \convex(x)\! > \!0.5$),
it can be effectively removed from the learning process.
This can easily happen when the convex enters a degenerate state (i.e.~$d_h\!\!=\!0 \:\: \forall h$).
Unfortunately these degenerate configurations are encouraged by the loss~\eq{loss_unique}.
We can prevent collapses by ensuring that each of them represents a certain amount of information (i.e. samples):
\begin{equation}
\loss{guide}(\omega) = \tfrac{1}{K} \sum_k \tfrac{1}{N} \sum_{\x \in \mathcal{N}_k^N} \| \convex_k(\x) - \object(\x) \|^2,
\label{eq:guide}
\end{equation}
where $\mathcal{N}^N_k$ is the subset of $N$ samples from the set $\x~\sim~\{\object\}$ with the smallest distance value~$\Phi_k(\x)$ from $\convex_k$.
In other words, each convex is responsible for  representing \textit{at least} the $N$ closest interior samples.

\paragraph{Localization loss (auxiliary)}
When a convex is far from interior points, the loss in \eq{guide} suffers from vanishing gradients due to the sigmoid function.
We overcome this problem by adding a loss with respect to $\mathbf{c}_k$, the translation vector of the $k$-th convex:
\begin{equation}
\loss{loc}(\omega) = \tfrac{1}{K} \sum_{\x \in \mathcal{N}^1_k} 
\| \mathbf{c}_k - x \|^2
\label{eq:loc}
\end{equation}

\paragraph{Observations}
Note that we supervise the indicator function $\convex$ rather than $\Phi$, as the latter \textit{does not} represent the signed distance function of a convex~(e.g. $\| \nabla \Phi(x) \| \neq 1$).
Also note how the loss in \eq{loss_decomp} is reminiscent of SIF~\cite[Eq.1]{genova_iccv19}, where the overall surface is modeled as a sum of \textit{meta-ball} implicit functions~\cite{metaball} -- which the authors call \enquote{structuring}.
The core difference lies in the fact that~\SIF{} models the surface of the object $\partial\object$ as an iso-level of the function \textit{post} structuring -- therefore, in most cases, the iso-surface of the individual primitives do not approximate the target surface, resulting in a slight loss of interpretability in the generated representation.


\subsection{Implementation details}
\label{sec:implementation}
To increase training speed, we sample a set of points on the ground-truth shape offline, precompute the ground truth quantities, and then randomly sub-sample from this set during our training loop.
For volumetric samples, we use the samples from~\OccNet{}, while for surface samples we employ the \enquote{near-surface} sampling described in~\SIF{}. Following~\SIF{}, we also tune down $\loss{approx}$ of \enquote{near-surface} samples by $0.1$.
We draw $100k$ random samples from the bounding box of $\object$ and $100k$ samples from each of $\partial\object$ to construct the points samples and labels. We use a sub-sample set (at training time) with $1024$ points for both sample sources. 
Although Mescheder~et~ al.~\cite{occnet_cvpr19} claims that using uniform volumetric samples are more effective than surface samples, we find that balancing these two strategies yields the best performance -- this can be attributed to the complementary effect of the losses in~\eq{loss_approx}~and~\eq{loss_decomp}.

\paragraph{Architecture details}
In all our experiments, we use the same architecture while varying the number of convexes and hyperplanes. For the \{Depth\}-to-3D task, we use 50 convexes each with 50 hyperplanes. For the RGB-to-3D task, we use 50 convexes each with 25 hyperplanes. 
Similar to~\OccNet{}, we use ResNet18 as the encoder $\encoder$ for both the \{Depth\}-to-3D and the RGB-to-3D experiments.
A fully connected layer then generates the latent code $\boldsymbol{\lambda} \in \R^{256}$ that is provided as input to the decoder $\decoder$.
For the decoder $\decoder$ we use a sequential model with four hidden layers with $(1024, 1024, 2048, |\mathbf{H}|)$ units respectively.
The output dimension is $|\mathbf{H}| = K(4+3H)$ where for each of the $K$ elements we specify a translation ($3$ DOFs) and a smoothness ($1$ DOFs). Each hyperplane is specified by the~(unit) normal and the offset from the origin ($3H$ DOFs).
In all our experiments, we use a batch of size $32$ and train with Adam with a learning rate of $10^{-4}$, $\beta_1=.9$, and $\beta_2=.999$. 
As determined by grid-search on the validation set, we set the weight for our losses $\{\loss{approx}:1.0, \loss{decomp}:0.1, \loss{unique}:0.001, \loss{guide}:0.01, \loss{loc}:1.0 \}$ and $\slope=75$.
\section{Experiments}
\label{sec:experiments}
%
We use the ShapeNet~\cite{chang2015shapenet} dataset in our experiments.
We use the same voxelization, renderings, and data split as in Choy~et.~al.~\cite{choy20163d}.
Moreover, we use the same multi-view depth renderings as~\cite{genova_iccv19} for our \{Depth\}-to-3D experiments, where we render each example from cameras placed on the vertices of a dodecahedron.
Note that this problem is a harder problem than 3D auto-encoding with point cloud input as proposed by~\OccNet{} and resembles more closely the single view reconstruction problem.
At training time we need ground truth inside/outside labels, so we employ the watertight meshes from~\cite{occnet_cvpr19} -- this also ensures a fair comparison to this method.
For the quantitative evaluation of semantic decomposition, we use labels from PartNet~\cite{mo2019partnet} and exploit the overlap with ShapeNet.

\begin{figure*}
\centering
\begin{overpic} 
[width=\linewidth]
{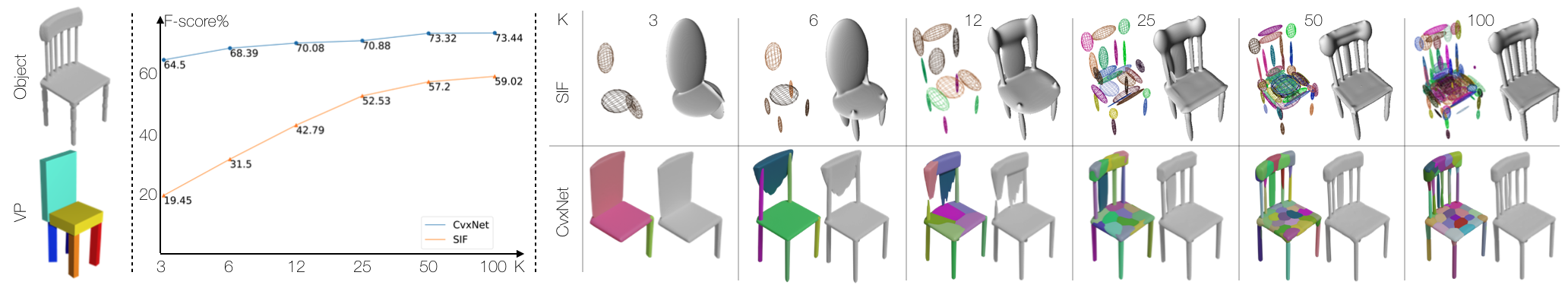}
\end{overpic}
\drawmyfilename{}
\caption{
\textbf{Analysis of accuracy vs. \# primitives} --
(left) The ground truth object to be reconstructed and the single shape-abstraction generated by \VP{}.
(middle) Quantitative evaluation (ShapeNet/Multi) of abstraction performance with an increase number of primitives -- the closer the curve is to the top-left, the better.
(right) A qualitative visualization of the primitives and corresponding reconstructions.
%
%
}
\label{fig:ratedistortion}
\end{figure*}
\begin{figure}[b]
\centering
\begin{overpic} 
[width=\linewidth]
{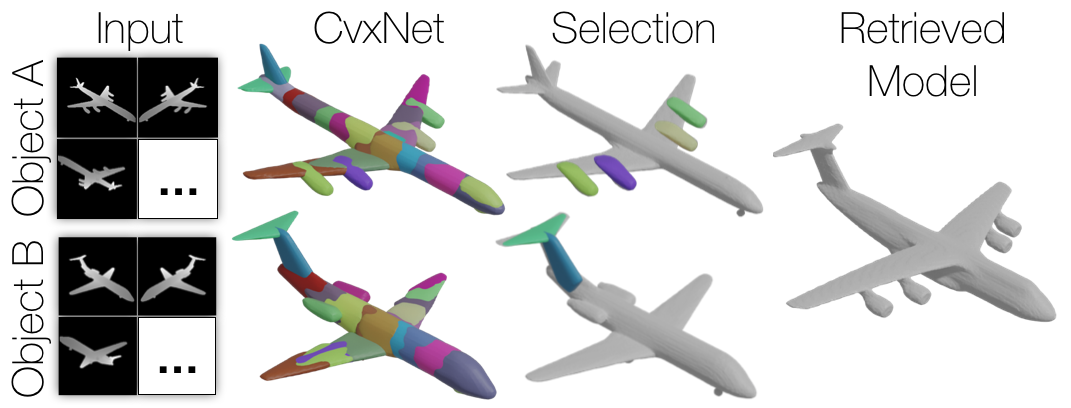}
\end{overpic}
\drawmyfilename{}
\caption{
\textbf{Part based retrieval} --
Two inputs (left) are first encoded into our CvxNet representation (middle-left), from which a user can select a subset of parts (middle-right).
We then use the concatenated latent code as an (incomplete) geometric lookup function, and retrieve the closest decomposition in the training database (right).
%
}
\label{fig:retrieval}
\end{figure}
\definecolor{partnet_arm}{RGB}{252,105,15}
\definecolor{partnet_base}{RGB}{39,147,33}
\definecolor{partnet_seat}{RGB}{202,16,31}
\definecolor{partnet_back}{RGB}{26,97,165}
\definecolor{darkgray}{rgb}{0.66, 0.66, 0.66}
\begin{figure}[b]
\begin{centering}
\begin{minipage}[c]{.33\linewidth}
  \begin{overpic} 
    [width=\linewidth]
    {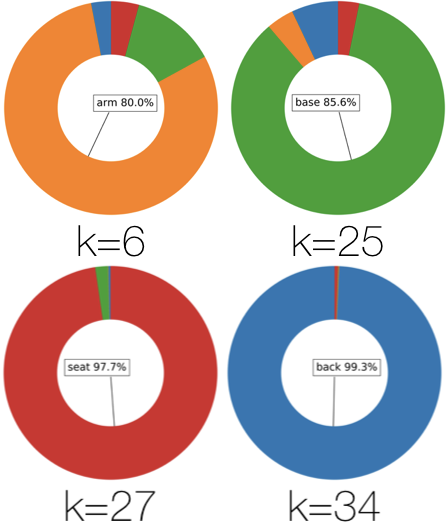}
    \end{overpic}
\end{minipage}
\begin{minipage}[c]{.62\linewidth}
  \resizebox{\linewidth}{!}{
  \begin{tabular}{c|cc|c}
  \multirow{2}{*}{Part} & \multicolumn{3}{c}{Accuracy} \\
                        & CvxNet & BAE & {\color{darkgray} BAE*} \\
  \midrule
  {\color{partnet_back} back} & \textbf{91.50\%} & 86.36\%          & {\color{darkgray} 91.81\%} \\
  {\color{partnet_arm} arm}   & 38.94\%          & \textbf{65.75\%} & {\color{darkgray} 71.32\%} \\
  {\color{partnet_base} base} & 71.95\%          & \textbf{88.46\%} & {\color{darkgray} 91.75\%} \\
  {\color{partnet_seat} seat} & \textbf{90.63\%} & 73.66\%          & {\color{darkgray} 92.91\%} \\
  \end{tabular}
  }
\end{minipage}
\end{centering}
\drawmyfilename{}
\vspace{\baselineskip}
\caption{
\textbf{Abstraction} -- 
(left) The distribution of partnet labels within each convex ID (4 out of 50).
(right) The classification accuracy for each semantic part when using the convex ID to label each point.
BAE~\cite{chen2019bae_net} is a baseline for unsupervised part segmentation. Finally, BAE* is the supervised version of BAE.
%
%
}
\label{fig:partnet}
\end{figure}


\paragraph{Methods}
We quantitatively compare our method to a number of self-supervised algorithms with different characteristics.
First, we consider~\VP{} that learns a parsimonious approximation of the input via (the union of) oriented boxes.
We also compare to the Structured Implicit Function~\SIF{} method that represents solid geometry as an iso-level of a sum of weighted Gaussians; like \VP{}, and in contrast to \OccNet{}, this methods provides an \textit{interpretable} encoding of geometry. 
Finally, from the class of techniques that \textit{directly} learn \textit{non-interpretable} representations of implicit functions, we select~\OccNet{},~\PtoM{}, and~\AtlasNet{}; in contrast to the previous methods, these solutions do not provide any form of shape decomposition.
As \OccNet{} only report results on RGB-to-3D tasks, we extend the original codebase to also solve \{Depth\}-to-3D tasks. We follow the same data  pre-processing used by~\SIF{}.


\paragraph{Metrics}
With $\hat\object$ and $\partial\hat\object$ we respectively indicate the indicator and the surface of the \textit{union} of our primitives.
We then use three quantitative metrics to evaluate the performance of 3D reconstruction:
%
\CIRCLE{1}
The \textit{Volumetric~IoU}; 
note that with $100K$ uniform samples to estimate this metric, our estimation is more accurate than the $32^3$ voxel grid estimation used by~\cite{choy20163d}.
\CIRCLE{2} The \textit{Chamfer-L1} distance, a smooth relaxation of the symmetric Hausdorff distance measuring the average between reconstruction \textit{accuracy} $\expect_{\bo \sim \partial\object} [\min_{\hat \bo \in \partial\hat\object} \|\bo - \hat\bo\|]$ and  \textit{completeness} $\expect_{\hat \bo \sim \partial\hat\object} [\min_{\bo \in \partial\object} \|\hat\bo - \bo\|]$~\cite{multirobot}.
\CIRCLE{3}
Following the arguments presented in~\cite{tatarchenko2019single}, we also employ \textit{F}-score to quantitatively assess performance. It can be understood as \enquote{the percentage of correctly reconstructed surface}.

\begin{figure*}[ht]
\centering
\begin{overpic} 
[width=\linewidth]
{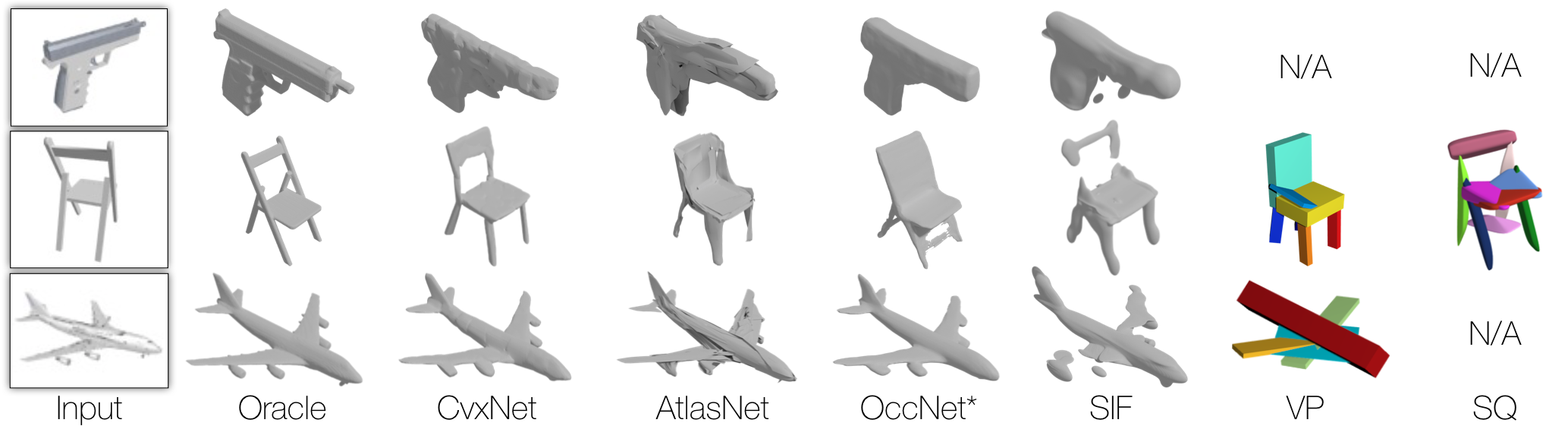}
\end{overpic}
\drawmyfilename{}
\caption{
\textbf{ShapeNet/Multi} -- 
Qualitative comparisons to \SIF{}, \AtlasNet{}, \OccNet{}, \VP{} and \SQ{}; on RGB Input, while VP uses voxelized, and SQ uses a point-cloud input.
($^*$Note that the \OccNet{} results are post-processed with smoothing).
%
%
}
\label{fig:shapenet}
\end{figure*}
\begin{table*}
\begin{center}
\setlength\tabcolsep{3pt} 
\resizebox{\textwidth}{!}{
\begin{tabular}{l|ccc|ccc|ccc} 
\toprule
\multirow{2}{*}{Category} & \multicolumn{3}{c}{IoU}                    & \multicolumn{3}{c}{Chamfer-$L_1$ }                 & \multicolumn{3}{c}{F-Score}                \\
                          & OccNet          & SIF    & Ours            & OccNet          & SIF            & Ours            & OccNet          & SIF   & Ours             \\ 
\midrule
airplane                  & 0.728           & 0.662                  & \textbf{0.739}  & 0.031           & 0.044          & \textbf{0.025}  & 79.52           & 71.40 & \textbf{84.68}   \\
bench                     & \textbf{0.655}  & 0.533                  & 0.631           & \textbf{0.041}  & 0.082          & 0.043           & 71.98           & 58.35 & \textbf{77.68}   \\
cabinet                   & \textbf{0.848}  & 0.783                  & 0.830           & 0.138           & 0.110          & \textbf{0.048}  & 71.31           & 59.26 & \textbf{76.09}   \\
car                       & \textbf{0.830}  & 0.772                  & 0.826           & 0.071           & 0.108          & \textbf{0.031}  & 69.64           & 56.58 & \textbf{77.75}   \\
chair                     & \textbf{0.696}  & 0.572                  & 0.681           & 0.124           & 0.154          & \textbf{0.115}  & 63.14           & 42.37 & \textbf{65.39}   \\
display                   & \textbf{0.763}  & 0.693                  & 0.762           & 0.087           & 0.097          & \textbf{0.065}  & 63.76           & 56.26 & \textbf{71.41}   \\
lamp                      & \textbf{0.538}  & 0.417                  & 0.494           & 0.678           & \textbf{0.342} & 0.352           & \textbf{51.60}  & 35.01 & 51.37            \\
speaker                   & \textbf{0.806}  & 0.742                  & 0.784           & 0.440           & 0.199          & \textbf{0.112}  & 58.09           & 47.39 & \textbf{60.24}   \\
rifle                     & 0.666           & 0.604                  & \textbf{0.684}  & 0.033           & 0.042          & \textbf{0.023}  & 78.52           & 70.01 & \textbf{83.63}   \\
sofa                      & \textbf{0.836}  & 0.760                  & 0.828           & 0.052           & 0.080          & \textbf{0.036}  & 69.66           & 55.22 & \textbf{75.44}   \\
table                     & \textbf{0.699}  & 0.572                  & 0.660           & 0.152           & 0.157          & \textbf{0.121}  & 68.80           & 55.66 & \textbf{71.73}   \\
phone                     & \textbf{0.885}  & 0.831                  & 0.869           & 0.022           & 0.039          & \textbf{0.018}  & 85.60           & 81.82 & \textbf{89.28}   \\
vessel                    & \textbf{0.719}  & 0.643                  & 0.708           & 0.070           & 0.078          & \textbf{0.052}  & 66.48           & 54.15 & \textbf{70.77}   \\ 
\midrule
mean                      & \textbf{0.744}  & 0.660                  & 0.731           & 0.149           & 0.118          & \textbf{0.080}  & 69.08           & 59.02 & \textbf{73.49}   \\
\bottomrule
\multicolumn{10}{c}{\textbf{\{Depth\}-to-3D}}
\end{tabular}
\hspace{.1in}
\begin{tabular}{l|ccccc|ccccc|cccc} 
\toprule
\multirow{2}{*}{Category} & \multicolumn{5}{c}{IoU}                                      & \multicolumn{5}{c}{Chamfer-$L_1$ }                                            & \multicolumn{4}{c}{F-Score}                                   \\
                          & P2M   & AtlasNet & OccNet          & SIF   & Ours            & P2M   & AtlasNet        & OccNet          & SIF             & Ours            & AtlasNet        & OccNet           & SIF   & Ours             \\ 
\midrule
airplane                  & 0.420 & -        & 0.571           & 0.530 & \textbf{0.598}   & 0.187 & 0.104           & 0.147            & 0.167 & \textbf{0.093}  & 67.24           & 62.87            & 52.81 & \textbf{68.16}   \\
bench                     & 0.323 & -        & \textbf{0.485}  & 0.333 & 0.461            & 0.201 & 0.138           & 0.155            & 0.261 & \textbf{0.133}  & 54.50           & \textbf{56.91}   & 37.31 & 54.64            \\
cabinet                   & 0.664 & -        & \textbf{0.733}  & 0.648 & 0.709            & 0.196 & 0.175           & 0.167            & 0.233 & \textbf{0.160 } & 46.43           & \textbf{61.79}   & 31.68 & 46.09            \\
car                       & 0.552 & -        & \textbf{0.737}  & 0.657 & 0.675            & 0.180 & 0.141           & 0.159            & 0.161 & \textbf{0.103}  & 51.51           & \textbf{56.91}   & 37.66 & 47.33            \\
chair                     & 0.396 & -        & \textbf{0.501}  & 0.389 & 0.491            & 0.265 & \textbf{0.209 } & 0.228            & 0.380 & 0.337           & 38.89           & \textbf{42.41}   & 26.90 & 38.49            \\
display                   & 0.490 & -        & 0.471           & 0.491 & \textbf{0.576}   & 0.239 & \textbf{0.198}  & 0.278            & 0.401 & 0.223           & \textbf{42.79}  & 38.96            & 27.22 & 40.69            \\
lamp                      & 0.323 & -        & \textbf{0.371}  & 0.260 & 0.311            & 0.308 & \textbf{0.305}  & 0.479            & 1.096 & 0.795           & 33.04           & \textbf{38.35}   & 20.59 & 31.41            \\
speaker                   & 0.599 & -        & \textbf{0.647}  & 0.577 & 0.620            & 0.285 & \textbf{0.245}  & 0.300            & 0.554 & 0.462           & 35.75           & \textbf{42.48}   & 22.42 & 29.45            \\
rifle                     & 0.402 & -        & 0.474           & 0.463 & \textbf{0.515 }  & 0.164 & 0.115           & 0.141            & 0.193 & \textbf{0.106}  & \textbf{64.22}  & 56.52            & 53.20 & 63.74            \\
sofa                      & 0.613 & -        & \textbf{0.680}  & 0.606 & 0.677            & 0.212 & 0.177           & 0.194            & 0.272 & \textbf{0.164}  & 43.46           & \textbf{48.62}   & 30.94 & 42.11            \\
table                     & 0.395 & -        & \textbf{0.506}  & 0.372 & 0.473            & 0.218 & 0.190           & \textbf{0.189 }  & 0.454 & 0.358           & 44.93           & \textbf{58.49}   & 30.78 & 48.10            \\
phone                     & 0.661 & -        & \textbf{0.720}  & 0.658 & 0.719            & 0.149 & 0.128           & 0.140            & 0.159 & \textbf{0.083}  & 58.85           & \textbf{66.09}   & 45.61 & 59.64            \\
vessel                    & 0.397 & -        & 0.530           & 0.502 & \textbf{0.552}   & 0.212 & \textbf{0.151 } & 0.218            & 0.208 & 0.173           & \textbf{49.87}  & 42.37            & 36.04 & 45.88            \\ 
\midrule
mean                      & 0.480 & -        & \textbf{0.571}  & 0.499 & 0.567            & 0.216 & \textbf{0.175 } & 0.215            & 0.349 & 0.245           & 48.57           & \textbf{51.75 }  & 34.86 & 47.36            \\
\bottomrule
\multicolumn{15}{c}{\textbf{RGB-to-3D} }
\end{tabular}

} 
\end{center}
\drawmyfilename{}
\vspace{1em}
\caption{
\textbf{Reconstruction performance on ShapeNet/Multi} --
We evaluate our method against \PtoM{}, \AtlasNet{}, \OccNet{} and 
\SIF{}. We provide in input either (left) a collection of depth maps or (right) a single color image.
For \AtlasNet{}, note that IoU cannot be measured as the meshes are not watertight.
We omit \VP{}, as it only produces a very rough shape decomposition.
%
%
}
\label{tbl:shapenet}
\end{table*}


\subsection{Abstraction -- Figure~\ref{fig:ratedistortion}, \ref{fig:retrieval}, \ref{fig:partnet}}
As our convex decomposition is learnt on a shape collection, the convexes produced by our decoder are in natural correspondence -- e.g. we expect the same $k$-th convex to represent the leg of a chair in the chairs dataset.
We analyze this quantitatively on the PartNet dataset~\cite{mo2019partnet}. We do so by verifying whether the $k$-th component is consistently mapped to the same PartNet part label; see \Figure{partnet}~(left) for the distribution of PartNet labels within each component.
We can then assign the most commonly associated label to a given convex to segment the PartNet point cloud, achieving a relatively high accuracy; see \Figure{partnet}~(right).
This reveals how our representation captures the semantic structure in the dataset.
We also evaluate our shape abstraction capabilities by varying the number of components and evaluating the trade-off between representation parsimony and reconstruction accuracy; we visualize this via Pareto-optimal curves in the plot of \Figure{ratedistortion}.
We compare with \SIF{}, and note that thanks to the generalized shape space of our model, our curve dominates theirs regardless of the number of primitives chosen.
We further investigate the use of natural correspondence in a part-based retrieval task. We first encode an input into our representation, allow a user to select a few parts of interest, and then use this (incomplete) shape-code to fetch the elements in the training set with the closest (partial) shape-code; see~\Figure{retrieval}.

\subsection{Reconstruction -- \Table{shapenet} and \Figure{shapenet}}
We quantitatively evaluate the reconstruction performance against a number of state-of-the-art methods given inputs as multiple depth map images (\{Depth\}-to-3D) and a single color image (RGB-to-3D); see \Table{shapenet}.
A few qualitative examples are displayed in \Figure{shapenet}.
We find that CvxNet is:
\CIRCLE{1} consistently better than other \textit{part decomposition} methods (SIF, VP, and SQ) which share the common goal of learning \textit{shape elements};
\CIRCLE{2} in general comparable to the state-of-the-art reconstruction methods;
\CIRCLE{3} better than the leading technique (\OccNet{}) when evaluated in terms of F-score, and tested on multi-view depth input.
Note that~\SIF{} first trains for the template parameters on (\{Depth\}-to-3D) with a reconstruction loss, and then trains the RGB-to-3D image encoder with a parameter regression loss; conversely, our method trains both encoder and decoder of the RGB-to-3D task from \textit{scratch}.

\subsection{Ablation studies}
\label{sec:ablation}
We summarize here the results of several ablation studies found in the \textbf{supplementary material}.
Our analysis reveals that the method is relatively insensitive to the dimensionality of the bottleneck $|\latent|$.
We also investigate the effect of varying the number of convexes $K$ and number of hyperplanes $H$ in terms of reconstruction accuracy and inference/training time.
Moreover, we quantitatively demonstrate that using signed distance as supervision for $\loss{approx}$ produces significantly worse results and at the cost of slightly worse performance we can collapse $\loss{guide}$ and $\loss{loc}$ into one.
Finally, we perform an ablation study with respect to our losses, and verify that each is beneficial towards effective learning.
\section{Conclusions}
We propose a differentiable representation of convex primitives that is amenable to learning. The inferred representations are \textit{directly usable} in graphics/physics pipelines; see \Figure{teaser}.
Our self-supervised technique provides more detailed reconstructions than very recently proposed part-based techniques~(\SIF{} in \Figure{ratedistortion}), and even consistently outperforms the leading reconstruction technique on multi-view input~(\OccNet{} in \Table{shapenet}).
In the future we would like to generalize the model to be able to predict a variable number of parts~\cite{tulsiani_cvpr17}, understand symmetries and modeling hierarchies~\cite{yu_cvpr19}, and include the modeling of rotations~\cite{tulsiani_cvpr17}.
Leveraging the invariance of hyperplane ordering, it would be interesting to investigate the effect of permutation-invariant encoders~\cite{acne}, or remove encoders altogether in favor of \textit{auto-decoder} architectures~\cite{park2019deepsdf}.

\paragraph{Acknowledgements}
We would like to acknowledge Luca Prasso and Timothy Jeruzalski for their help with preparing the rigid-body simulations, Avneesh Sud and Ke Li for reviewing our draft, and Anton Mikhailov, Tom Funkhouser, and Erwin Coumans for fruitful discussions.
\clearpage
{
    \small
    \bibliographystyle{ieee_fullname}
    \bibliography{bib_macros,bib_main}
}
\clearpage
\title{\textbf{CvxNet: Learnable Convex Decomposition} \\ Supplementary Material
}
\author{}
\date{}
\maketitle

\begin{figure}
\begin{overpic} 
[width=\linewidth]
{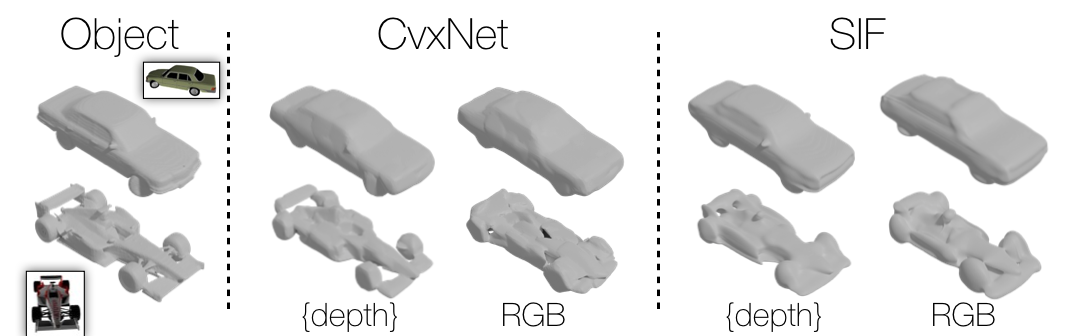}
\end{overpic}
\drawmyfilename{}
\caption{
\textbf{Depth vs. Color} --
A qualitative example illustrating the degradation in performance as we move from \{depth\} to a \enquote{weaker} RGB input.
In the top row we illustrate a model where the frequency of the surface is low. In this case, both \{depth\} input and RGB input can approximate the shape relatively accurately. 
In contrast, the bottom row shows when the ground truth surface has high frequency, the RGB input results lose many details while the \{depth\} input results remain accurate.
}
\label{fig:depthvsrgb}
\end{figure}
\begin{figure}
\begin{center}
\setlength\tabcolsep{3pt} 
\resizebox{0.6\linewidth}{!}{
\begin{tabular}{c|ccc}
\toprule
Losses & IoU & Chamfer-$L_1$ & F-Score \\
\midrule
Original & \textbf{0.731} & \textbf{0.080} & \textbf{73.49\%} \\
Merged & 0.720 & 0.087 & 71.62\% \\
\bottomrule
\end{tabular}
} 
\end{center}
\drawmyfilename{}
\vspace{\baselineskip}
\caption{
\textbf{Original Losses vs. Merged Losses} -- A quantitative comparison between the original losses and the merged version where guidance loss and localization loss are collapsed as described in \ref{sec:mergedloss}. At the cost of slightly worse performance, we can simplify the training objective.
%
%
}
\label{tbl:ablationmergeloss}
\end{figure}
\begin{figure}[t]
\begin{center}
\resizebox{\linewidth}{!}{
\begin{tabular}{c|c|cc|cc} 
\toprule
\multirow{2}{*}{Class} & \multirow{2}{*}{\# exemplars} & \multicolumn{2}{c|}{\textbf{\{Depth\}-to-3D } } & \multicolumn{2}{c}{\textbf{RGB-to-3D} }  \\
                       &                               & F\% Single & F\% Multi               & F\% Single  & F\% Multi            \\ 
\midrule
table                  & 5958                          & 70.97          & \textbf{71.71 }             & \textbf{50.29}  & 48.10                    \\
car                    & 5248                          & 77.11          & \textbf{77.75 }             & \textbf{51.34}  & 47.33                    \\
chair                  & 4746                          & 62.21          & \textbf{65.39 }             & 38.12           & \textbf{38.49}           \\
plane               & 2832                          & 83.68          & \textbf{84.68 }             & \textbf{75.19}  & 68.16                    \\
sofa                   & 2222                          & 67.89          & \textbf{75.44 }             & \textbf{43.36}  & 42.11                    \\
rifle                  & 1661                          & 82.73          & \textbf{83.63 }             & \textbf{69.62}  & 63.74                    \\
lamp                   & 1624                          & 46.46          & \textbf{51.37 }             & \textbf{32.49}  & 31.41                    \\
vessel                 & 1359                          & 65.71          & \textbf{70.77 }             & \textbf{48.44}  & 45.88                    \\
bench                  & 1272                          & 68.20          & \textbf{77.68 }             & \textbf{59.27}  & 54.64                    \\
speaker                & 1134                          & 50.37          & \textbf{60.24 }             & 28.07           & \textbf{28.45}           \\
cabinet                & 1101                          & 66.75          & \textbf{76.09 }             & 45.73           & \textbf{46.09}           \\
display                & 767                           & 61.66          & \textbf{71.41 }             & \textbf{40.31}  & 38.96                    \\
phone                  & 737                           & 84.93          & \textbf{89.28 }             & 63.58           & \textbf{66.09}           \\ 
\midrule
mean                   & 2359                          & 68.36          & \textbf{73.50 }             & \textbf{49.68 } & 47.65                    \\
\bottomrule
\end{tabular}
} 
\end{center}
\drawmyfilename{}
\vspace{\baselineskip}
\caption{
\textbf{Single-class vs. multi-class} --
It is interesting to note that training on single vs. multi class has a behavior different from what one would expect (i.e. overfitting to a single class is beneficial).
Note how training in multi-class on the \{Depth\}-to-3D input improves the reconstruction performance across the \textit{entire} benchmark.
Conversely, with RGB-to-3D input, single class training is beneficial in most cases.
We explain this by the fact that RGB inputs have complex texture which is stable within each class but not easily transferable across classes. Contrarily, local geometries learned from \{depth\} are agnostic to classes. 
%
}
\label{fig:singlevsmulti}
\end{figure}

\begin{figure}
\begin{center}
\setlength\tabcolsep{3pt} 
\resizebox{\linewidth}{!}{
\begin{tabular}{c|ccc}
\toprule
\diagbox{K}{H}  & 12 & 25 & 50 \\
\midrule
12 & 0.709 / 0.139 / \textbf{44} & 0.712 / 0.127 / 59 & 0.714 / 0.124 / 87 \\
25 & 0.717 / 0.100 / 60 & 0.720 / 0.099 / 92 & 0.721 / 0.096 / 153 \\
50 & 0.724 / 0.088 / 92 & 0.730 / 0.083 / 156 & \textbf{0.731} / \textbf{0.080} / 280 \\
\bottomrule
\end{tabular}  
} 
\end{center}
\drawmyfilename{}
\vspace{\baselineskip}
\caption{
\textbf{Ablation on model complexity} -- We analyze how the number of hyperplanes $H$ and number of convexes $K$ relate to mIoU / Chamfer-$L_1$ / Inference Time (ms).
The mIoU and Chamfer-$L_1$ are measured on the test set of ShapeNet (multi-class) with multi-view depth input.
We measure the inference time of a batch with $32$ examples and $2048$ points, which is equivalent to one forward propagation step at training time.
}
\label{fig:complexity}
\end{figure}
\begin{figure}
\begin{center}
\setlength\tabcolsep{3pt} 
\small
  \begin{tabular}{c|ccccc}
  \toprule
    \diagbox{Metric}{Loss}  & All & $-\loss{decomp}$ & $-\loss{unique}$ & $-\loss{guide}$ & $-\loss{loc}$ \\
\midrule
Vol-IoU & \textbf{0.567} & 0.558             & 0.545             & 0.551            & 0.558           \\
Chamfer & \textbf{0.245} & 0.308             & 0.313             & 0.335            & 0.618           \\
F\%     & \textbf{47.36} & 45.29             & 44.03             & 45.88            & 46.01    \\
\bottomrule
\end{tabular}
\end{center}
\drawmyfilename{}
\vspace{\baselineskip}
\caption{
\textbf{Ablation on losses} -- 
We test on ShapeNet multi with RGB input, where each column (from 2 to 5) removes ($-$) one loss term from the training. We observe that each loss can improve the overall performance.
}
\label{fig:ablationlosses}
\end{figure}
\begin{figure}
\begin{center}
\setlength\tabcolsep{3pt} 
\small
\begin{tabular}{c|ccc}
\toprule
Supervision & IoU & Chamfer-$L_1$ & F-Score \\
\midrule
Indicator & \textbf{0.747} & \textbf{0.035} & \textbf{83.22\%}\\
SDF & 0.650 & 0.045 & 73.08\% \\
\bottomrule
\end{tabular}
\end{center}
\drawmyfilename{}
\vspace{\baselineskip}
\caption{
\textbf{Ablation SDF Training vs. Indicator Training} -- We study the difference between learning signed distance functions(i.e., replacing $\object(\x)$ with signed distance in \eq{loss_approx} and removing $\sigmoid$ in \eq{indicator}) and learning indicator functions. Note how the performance would degenerate significantly when the model learns to predict signed distances.
}
\label{tbl:ablationsdf}
\end{figure}
\begin{figure}
\begin{minipage}[c]{.25\linewidth}
  \begin{center}
    \setlength\tabcolsep{3pt} 
    \resizebox{\linewidth}{!}{
      \begin{tabular}{c|c}
        $|\boldsymbol{\lambda}|$  & F-Score \\
        \midrule
        32 & 72.24\% \\
        64 & 73.31\% \\
        128 & 73.14\% \\
        256 & \textbf{73.49\%}
        \end{tabular}  
    } 
    \end{center}
\end{minipage}
\begin{minipage}[c]{.74\linewidth}
\caption{
\textbf{Ablation latent size} --
We study the performance of our architecture as we vary the number of latent dimensions in \Figure{cvxfam}. Note how the reconstruction performance remains relatively stable as we vary this hyper-parameter.
}
\label{tbl:ablationlambda}
\end{minipage}
\end{figure}
\begin{figure}
\begin{center}
\begin{overpic} 
[width=\linewidth]
{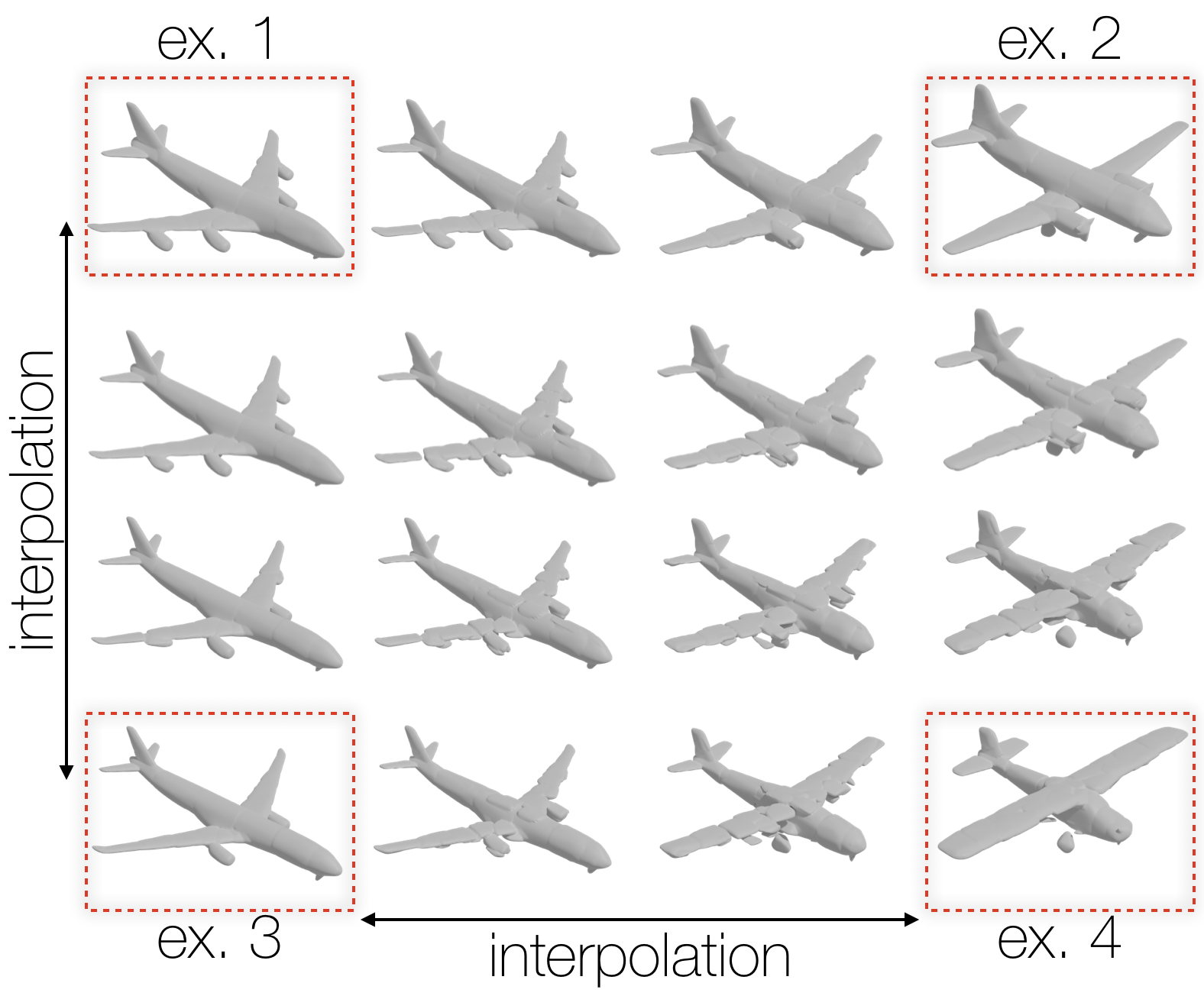}
\end{overpic}
\end{center}
\drawmyfilename{}
\caption{
\textbf{Shape interpolation} --
Analogously to \Figure{interp2d}, we encode the four models in the corners, and generate the intermediate results via linear interpolation. These results can be better appreciated in the \textbf{supplementary video} where we densely interpolate between the corner examples. 
%
%
}
\label{fig:interp3d}
\end{figure}


\section{Union of smooth indicator functions}
\label{sec:union}
\noindent

We define the smooth indicator function for the $k$-th object:
\begin{equation}
    \convex_k(\x) = \sigmoid_{\slope} (-\Phi_\smooth^k(\x)),
\end{equation}
where $\Phi_\smooth^k(\x)$ is the $k$-th object signed distance function. In constructive solid geometry the union of signed distance function is defined using the $\min$ operator.
Therefore the union operator for our indicator function can be written:
\begin{align}
\union\{ \convex_k (\x) \} &= \sigmoid_{\slope} (-\min_k \{\Phi_\smooth^k(\x)\}) \\
&= \sigmoid_{\slope} (\max_k \{-\Phi_\smooth^k(\x)\})
\label{eq:commutative}
\\
\nonumber
&= \max_k \{\sigmoid_{\slope}(-\Phi_\smooth^k(\x))\} = \max_k \{ \convex_k (\x) \}.
\end{align}
Note that the $\max$ operator is \textit{commutative} with respect to monotonically increasing functions allowing us to extract the $\max$ operator from the $\sigmoid_{\slope}(\cdot)$ function in \eq{commutative}.

\section{Merged guidance loss and localization loss}
\label{sec:mergedloss}
While the guidance loss \eq{guide} and localization loss \eq{loc} are designed by different motivations, they are inherently consistent in encouraging the convex elements to reside close to the ground truth. Therefore, we propose an alternative training strategy to merge these two losses into a single one:
\begin{equation}
\loss{merged} = \tfrac{1}{K} \sum_k \tfrac{1}{N} \sum_{\x \in \mathcal{N}_k} \| \text{ReLU}(\Phi_k(\x)) \|^2,
\label{eq:mergedloss}
\end{equation}
%
The idea is that while $\Phi_k(\x)$ is approximate, it can still be used to get weak gradients.
Essentially, this means that each convex needs to \enquote{explain} the N closest samples of $\object$.

\section{Proof of auxiliary null-space loss}
Consider a collection of hyperplanes $\mathbf{h}=\{(\n_h, d_h)\}$
where each pair $(\n_h, d_h)$ denotes
$\halfspace_h(\x)=\n_h \cdot \x + d_h$.
We like to show that amongst all possible translations of this
collection, there is a unique one that minimizes $\loss{unique} = \tfrac{1}{H} \sum_h \halfspace_h(\x)^2$.
We can rewrite this sum as
\begin{align}
    \| (\halfspace_h(\x))_h \|^2
    = \| (d_h + \n_h \cdot \x)_h \|^2
    = \| \mathbf{d_h} + N \x \| ^2
\end{align}
where $\mathbf{d_h}{=}(d_h)_h$ is a vector with entries $d_h$, and $N$ is the matrix formed by $\n_h$ as columns. However it is well known that the above minimization has a unique solution, and the minimizing solution can be computed explicitly (e.g. by using the Moore-Penrose inverse of $N$).

We also note that $\loss{unique}$ has a nice geometric description.
In fact, minimizing $\loss{unique}$ is essentially ``centring'' the convex body at the origin.
To see this we first prove the following:

\begin{observation}
Let $\halfspace_h(\x)=\n_h\cdot\x+d_h$ be a hyperplane and let $\x_0$ be any point.
Then the distance between $\x_0$ and the hyperplane $\halfspace_h$ is given by ${|\halfspace_h(\x_0)|} / {\|\n_h\|}$.
\end{observation}
\begin{proof}
  Let $X$ be any point on the hyperplane $h$, so $\halfspace_h(X)=0$.
  Then the distance of $\x_0$ to $\halfspace_h$ is the norm of the projection of $X-\x_0$
  onto the norm of $\halfspace_h$, $\n_h$.
  Using standard formulas we get, as desired:
  \begin{align}
      \|\proj_{\n_h}(X-\x_0)\| & = \left\| \frac{(\x_0 - X)\cdot \n_h}{\|\n_h\|^2} \n_h \right\| \\
      & = \frac{\|\n_h \cdot x_0 - \n_h \cdot X\|}{\|\n_h\|} \\
      & = \frac{\|\n_h \cdot x_0 + d_h\|}{\|\n_h\|}
  \end{align}
\end{proof}
Note that if we assume that $\n_h$ are all normalized, then $|\halfspace_h(\x_0)|$ gives us the distance
to the hyperplane $h$ for each hyperplane $h$.
Therefore, given a convex polytope defined by collection of hyperplanes $\mathbf{h} = \{(\n_h, d_h)\}$,
the $\sum_h |\halfspace_h(\x)|^2$ is the sum of squares of distance of $\x$ to the sides of the polytope.
Also note that if we translate $\mathbf{h}$ by $\x_0$, then we get the new set of hyperplane $\mathbf{h}' = \{(\n_h, d_h + \n_h\cdot \x_0)\} = \{(\n_h, \halfspace_h(\x_0))$. Therefore, 
\begin{equation}
\loss{unique}=\sum_h \|\halfspace_{h}(\mathbf{0})\|^2
\end{equation}
can be interpreted as the square distance of the origin to the collection of hyperplanes $\mathbf{h}$.

\begin{figure*}[t]
\centering
\begin{overpic} 
[width=\linewidth]
{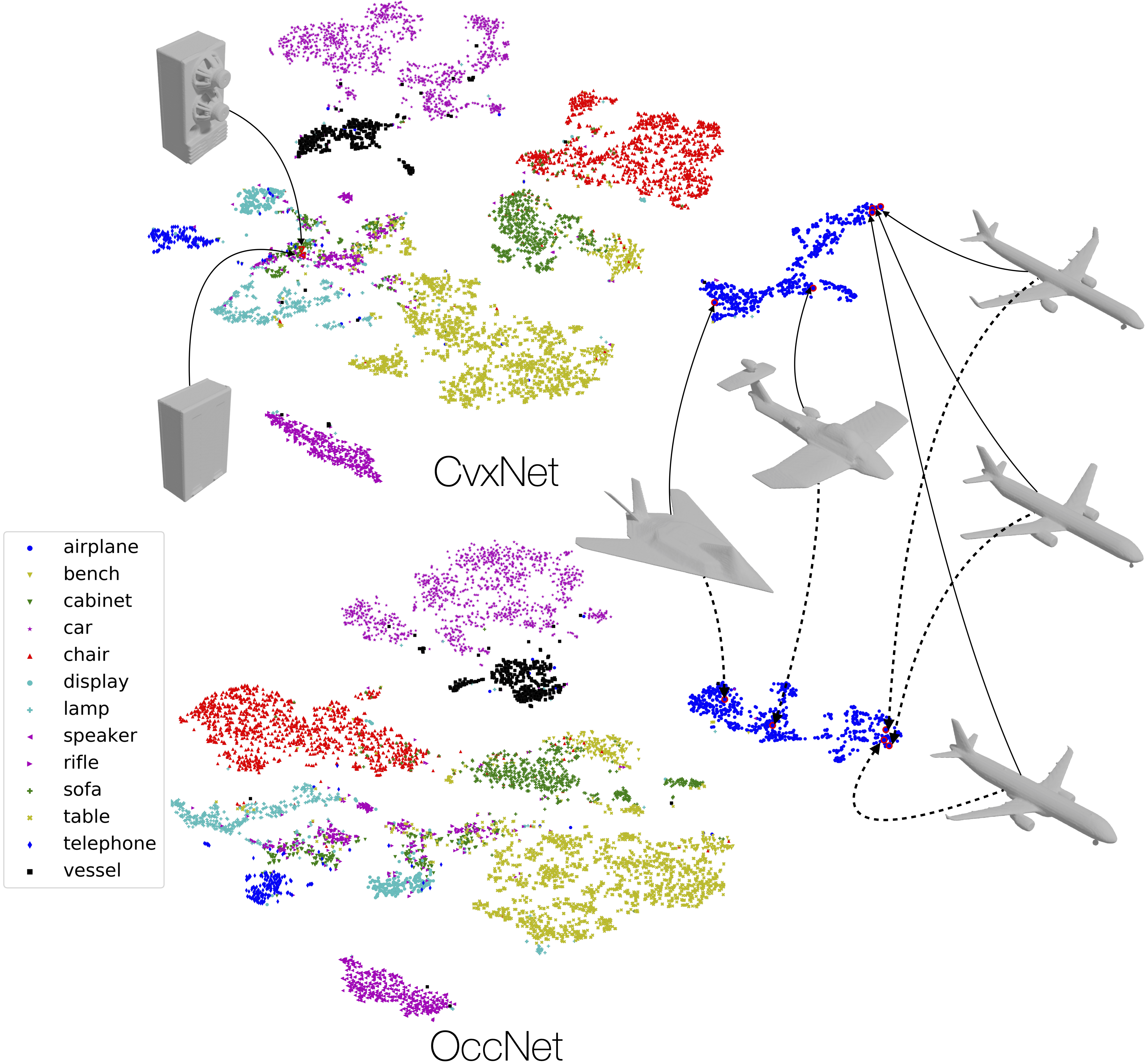}
\end{overpic}
\drawmyfilename{}
\caption{
\textbf{tSNE embedding} --
We visualize the CvxNet~(top) vs. the OccNet~(bottom) tSNE latent space in 2D. We also visualize several examples with arrows taking them back to the CvxNet atlas and dashed arrows pointing them to the OccNet atlas.
Notice how \CIRCLE{1} nearby (distant) samples within the same class have a similar (dissimilar) geometric structure, and 
\CIRCLE{2} the overlap between cabinets and speakers is meaningful as they both exhibit a cuboid geometry.
Our interactive exploration of the t-SNE space revealed that our method produces meaningful embeddings, which are comparable to the ones from \OccNet{}.
}
%
\end{figure*}


\end{document}